\documentclass[letterpaper, 10 pt, conference]{ieeeconf}  %

\IEEEoverridecommandlockouts                              %

\usepackage{graphics} %
\usepackage{epsfig} %
\usepackage{amsmath,mathtools,amsfonts} 
\usepackage{amssymb}
\usepackage{accents}
\usepackage{verbatim}
\usepackage{kantlipsum}
\usepackage{color}
\usepackage{xcolor}
\usepackage{cite}
\usepackage{booktabs}

\usepackage{tabularx}
\usepackage{subcaption}

\usepackage{algorithm}
\usepackage{algpseudocode}

\usepackage{multirow}

\newtheorem{thm}{Theorem}

\newtheorem{cor}{Corollary}

\newtheorem{rem}{Remark}

\newtheorem{problem}{Problem}

\newcommand{\changes}[1]{\textcolor{black}{#1}}
\newcommand{\newchanges}[1]{\textcolor{black}{#1}}

\title{\LARGE \bf
Meta-Learning Augmented MPC for Disturbance-Aware Motion Planning and Control of Quadrotors%
}

\author{Dženan Lapandić$^{1}$, Fengze Xie$^{2}$, Christos K. Verginis$^{3}$, \\ Soon-Jo Chung$^{2}$, Dimos V. Dimarogonas$^{1}$ and Bo Wahlberg$^{1}$%
\thanks{*This work was supported by the Wallenberg AI, Autonomous Systems and Software Program (WASP), the Swedish Research Council and Knut and Alice Wallenberg Foundation (KAW).}%
\thanks{$^{1}$Dženan Lapandić, Dimos V. Dimarogonas and Bo Wahlberg are with  Division of Decision and Control Systems, KTH Royal Institute of Technology, Stockholm, Sweden.   {\tt\small lapandic,dimos,bo@kth.se}}%
\thanks{$^{2}$Fengze Xie and Soon-Jo Chung are with Division of EAS, Caltech, Pasadena, CA, USA   {\tt\small fxxie,sjchung@caltech.edu}}%
\thanks{$^{3}$Christos K. Verginis is with Division of Signals and Systems, Department of Electrical Engineering, Uppsala University,Uppsala, Sweden. {\tt\small christos.verginis@angstrom.uu.se}}%
}

\begin{document}

\maketitle
\thispagestyle{empty}
\pagestyle{empty}

\begin{abstract}

A major challenge in autonomous flights is unknown disturbances, which can jeopardize safety and cause collisions, especially in obstacle-rich environments.
This paper presents a disturbance-aware motion planning and control framework for autonomous aerial flights.
The framework is composed of two key components: a disturbance-aware motion planner and a tracking controller.
The motion planner consists of a predictive control scheme and an \changes{online-adapted} learned disturbance model.
The tracking controller\changes{, developed} using contraction control \changes{methods, ensures} safety bounds on the \changes{quadrotor's behavior near} obstacles with respect to the motion plan.
The algorithm is tested in simulations with a quadrotor facing strong crosswind and ground-induced disturbances.

\end{abstract}

\section{Introduction}
Enhancing the autonomy of unmanned aerial vehicles (UAVs) \changes{has made} safe autonomous landing in harsh environments a \changes{key} research challenge. This capability is relevant in various domains and applications, including air mobility, search and rescue, and drone delivery \cite{bauranov2021designing,rajabi2021drone,lyu2023unmanned}.  
Developing robust quadrotor landing algorithms is challenging due to disturbances and safety-critical constraints near obstacles. 
Therefore, planning and control algorithms \changes{must account for} these disturbances and their effects on UAV \changes{performance}.

Rotor-based aircraft experience increased thrust near the ground due to reduced downwash, known as the ground effect, first documented for helicopters in \cite{cheeseman1955effect}. 
\changes{Modeling ground effects is complex, and neglecting them can pose significant} safety risks.
Accurately identifying or learning these disturbance models can lead to smoother, safer UAV landings. 
Ground effect disturbances, a type of interaction-produced disturbance, can be modeled using neural networks that take the relative position of the UAV from the ground as input \cite{shi2019neural}. 
\changes{Incorporating such models into motion planning enables UAVs} to find and execute optimal trajectories \changes{that account} for ground effect disturbances.

In this paper, we focus on augmenting the nominal quadrotor dynamics with a learned disturbance model inside the MPC scheme for predicting future behaviour. 
The disturbance model is acquired and refined through meta-learning, \changes{defined here} as an online adaptation of the offline\changes{-trained} disturbance model to the observed environmental conditions.
Our meta-learning \changes{approach} leverages deep neural networks, with the final layer dynamically adapted online through adaptive control mechanisms \cite{aastrom2013adaptive}. 
Since the application is safety-critical, \changes{collecting} data to learn the representation must be done efficiently and safely. 
\changes{During each MPC iteration}, the meta-learning algorithm updates the parameter estimates and covariance matrix, which are then applied in the next prediction step to \changes{refine} the disturbance model.
A low-level contraction-based controller \changes{complements} the feedforward MPC control action, \changes{ensuring} convergence to the planned trajectory.
Contraction theory \cite{lohmiller1998contraction} provides performance bounds \textit{a priori}, with respect to the desired trajectory.
The planner can utilize this information to guarantee collision-free behaviour near obstacles.
Thus, the proposed framework achieves disturbance-aware planning and control with theoretical \changes{safety} guarantees.

\begin{figure}[t!]
\centering
  \includegraphics[width=0.75\columnwidth]{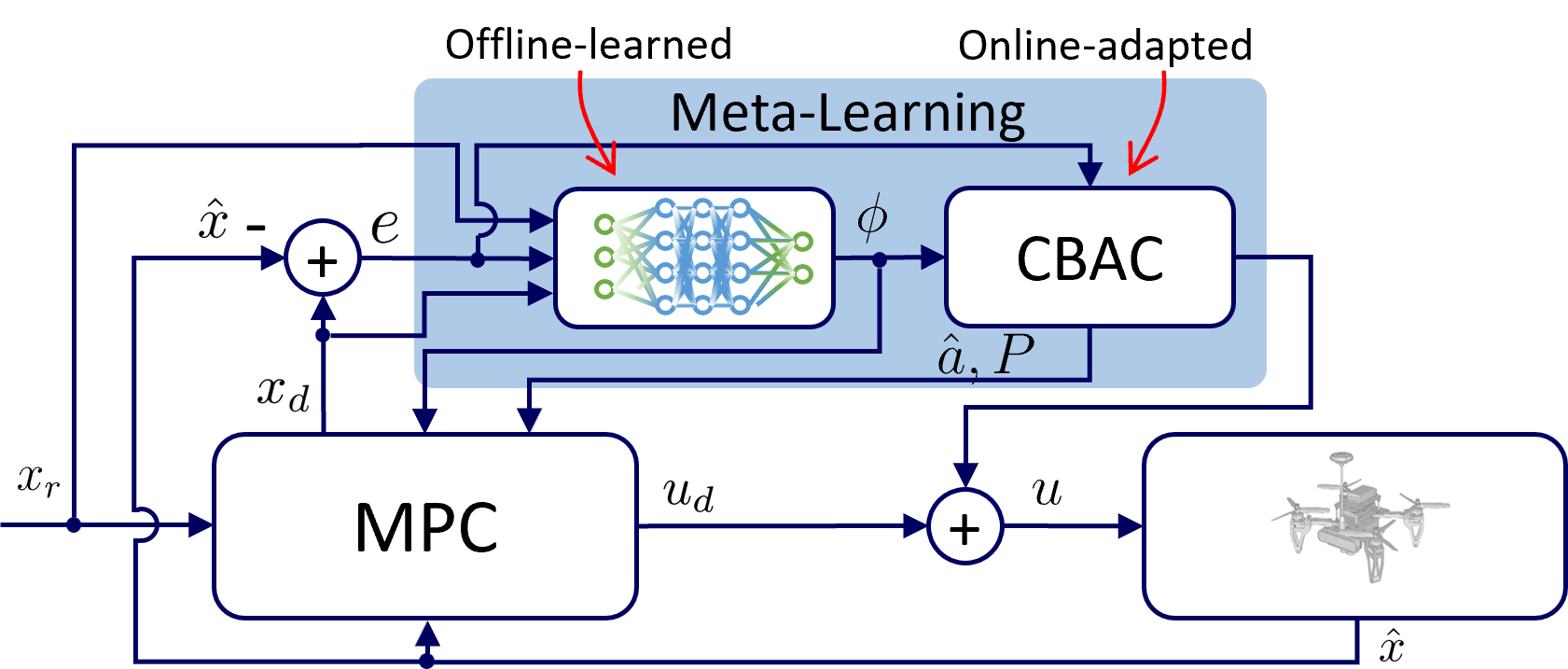}
\caption{Block diagram of the proposed disturbance-aware motion planning and control algorithm.}
\label{fig:diagram}\vspace{-0.75cm}
\end{figure}

\setlength{\voffset}{0.2cm}
The main contributions of this paper are 
{
\renewcommand{\theenumi}{\roman{enumi}}
\begin{enumerate}
    \item the augmentation of the model predictive control (MPC) with the learned model of disturbances within the proposed meta-learning framework,
    \item stability guarantees for the models with approximated state-dependent coefficients, 
    \item theoretical considerations on chance-constrained upper bounds for safety.
\end{enumerate}
}
\subsection{Related Work}
Recent \changes{research} on quadrotor control and disturbance handling \changes{often} uses optimization-based controllers \changes{like} model predictive control (MPC) with \changes{adaptive methods} for model mismatch \cite{hanover2021performance}. 
In \cite{torrente2021data}, \changes{Gaussian processes model} aerodynamic effects and propagate corrected dynamics \changes{within} MPC, outperforming \changes{linear compensation approaches} \cite{faessler2017differential}. 
\changes{Disturbances are categorized} as matched and unmatched in \cite{sinha2022adaptive}, with online estimation and adaptation countering matched disturbances \changes{while ensuring} unmatched ones remain bounded. 
Ground-effect disturbances \cite{cheeseman1955effect,khromov2008ground,he2017modeling,danjun2015autonomous} \changes{have been effectively represented and cancelled} using \changes{spectrally-normalized deep neural networks (DNNs)} \cite{shi2019neural}. 
The characteristics of interaction-produced disturbances have been studied in \cite{shi2020neuralswarm,shi2021neuralswarm2}.
In \cite{oconnel2022neural, 10611562}, authors use meta-learning to combine offline learning and online adaptation to cancel the wind disturbances represented by a learned deep neural network and a set of linear coefficients adapted online for the current wind conditions.
Such a decomposition effectively represents the unknown wind dynamics. Furthermore, meta-learning approaches have been explored in \cite{richards2021adaptive,harrison2020meta,shi2021meta,mckinnon2021meta, verginis2023non, 10706036}.

Note that a similar study on using interaction-aware disturbances for motion planning has been conducted in \cite{shi2021neuralswarm2}. However, the main differences to this work are that we now consider the nonlinear model of quadrotor dynamics instead of double-integrator dynamics and use the control inputs provided as the outcome of the optimal control problem as a feedforward input to the low-level controller.

The paper is organized as follows. The problem formulation is given in Sec.~\ref{sec:prob_form}, and the meta-learning augmented MPC is presented in Sec.~\ref{sec:mlocp}. Finally, Sec.~\ref{sec:results} presents the simulation results, and Sec.~\ref{sec:conclusion} concludes the paper.
\subsection{Notation}
For a matrix $A\in\mathbb{R}^{n\times n}$, $A\succ0$ denotes that $A$ is positive definite, \newchanges{$A\succeq0$ that $A$ is positive semi-definite,} and $\lambda_{\textup{min}}(A)$ and $\lambda_{\textup{max}}(A)$ are minimal and maximal eigenvalues of $A$, respectively.
\newchanges{Also, $2\textup{sym}(A) = A + A^T$.}

\section{Problem Formulation}\label{sec:prob_form}
We consider the quadrotor UAV model:
\begin{subequations}\label{eq:system}
\begin{align}
    \dot{p} &= v,\\
    m\dot{v} &=R ({\eta})f_T + f_{dist} - mg,\label{eq:1b} \\
    \dot{\eta} &= R_T ({\eta})\omega,\\
    J \dot{\omega} &=  J \omega\times \omega + \tau_{dist} + \tau_u \label{eq:1d}
\end{align}%
\end{subequations}%
where $x = [ p^T, v^T, \eta^T, \omega^T]^T \in \mathcal{X}=\mathbb{R}^{12}$, $p \in \mathbb{R}^3$ is the position in the inertial frame, $v \in \mathbb{R}^3$ is the linear velocity, $\eta = [\phi,\theta,\psi]^T \in \mathbb{T} = (-\frac{\pi}{2},\frac{\pi}{2})\times(-\frac{\pi}{2},\frac{\pi}{2})\times (-\pi,\pi)$ is the vector of Euler angles representing the attitude (roll, pitch, yaw angles), and $\omega = [\omega_\phi,\omega_\theta,\omega_\psi]^T$ is the angular velocity, expressed in the body frame; 
$f_{T} = [0,0,f_u]^T$ is the controlled thrust, and $\tau_u$ is the controlled torque;
$R: \mathbb{T} \to SO(3)$ is the rotation matrix from the body to the inertial frame, and $SO(3)$ is the special orthonormal group %
in 3D and $I_n \in \mathbb{R}^{n\times n}$ is the identity matrix; 
Furthermore, $R_{T}:\mathbb{T}\to\mathbb{R}^{3\times 3}$ is the mapping from the angular velocity to the time derivatives of the Euler angles. 
Wind disturbances, ground effects, unmodeled aerodynamic forces, and \changes{other external disturbances are captured by $f_{dist}$. Since the disturbance torque $\tau_{dist}$ is two orders of magnitude smaller than other feedback terms in a standard attitude controller loop \cite{shi2021neuralswarm2}, we focus on learning $f_{dist}$ alone.}
Finally, $m$ and $J$ are the mass and positive definite inertia matrix of the UAV.

Because of the assumption on additive disturbances as in $\eqref{eq:1b}$, the quadrotor model can be rewritten in the following form
\begin{equation}\label{eq:system_input_affine}
    \dot{x} = f(x) + B(x)u + f_d
\end{equation}
\changes{where $f_d=[0^T,f_{dist}^T,0^T,0^T]^T$ because we assume $\tau_{dist}~\approx~0$}.
We consider the quadrotor tasked with landing on a moving platform with the reference state $x_r(t)$ being derived from the motion of the platform.
\begin{problem}
    Given the state estimate of the moving landing platform $x_r(t)\in \mathcal{X}$, $t \geq 0$, derive a trajectory $x:[0,t_f]\rightarrow \mathcal{X}$ and a control input $u=[f_u,\tau_u^T]^T$, with control thrust $f_u$ and torque $\tau_u$
    \newchanges{such} that the state of quadrotor $x(t_f)$ at $t_f$ is in the goal region $\mathcal{X}_{\textup{goal}}(x_r(t_f))$ determined by the state of the landing platform and state and control constraints $x\in \mathcal{X}$, $u\in \mathcal{U}$ \changes{are satisfied} for all $t$.
\end{problem}

\section{Meta-learning Augmented MPC}\label{sec:mlocp}
The interaction-produced disturbances, mainly generated by the ground effect, can also be modelled with a neural network that, as input, takes the relative position from the ground surface \cite{shi2019neural}.
In this work, we introduce a disturbance model that incorporates the position of the targeted ground surface, setting it apart from the approach in \cite{oconnel2022neural}
\begin{equation}
    f_d(x,x_r,w) \approx \phi(x,x_r,\Theta)a(w) \label{eq:f_d_modelguess}
\end{equation}
where \changes{$a(w)$ is an unknown parameter that depends on $w \in \mathbb{R}^m$ which is} an unknown hidden state representing the underlying environmental conditions and can also be time-varying, and $\phi$ is a neural network with parameters $\Theta$. \changes{Note that $\phi$ has non-zero components only in the linear velocity term \eqref{eq:1b}, so our $\phi = [0, \phi_v, 0]$.}
The function $\phi$ constitutes a basis function in the meta-learning approach that is invariant to the specific environment conditions. 
\changes{The specific conditions are tackled online by dynamically adapting the parameter $a$ in the disturbance model which, practically,  can be seen as changing the final layer of the learned neural network model.}

\subsection{Neural Network Model of Disturbances}\label{ssec:nn_model_of_disturbances}

The first stage is an offline training of a neural network based on the synthetically generated dataset $\mathcal{D}_{\textup{meta}}=\{D_1,...,D_M \}$ consisting of $M$ different environmental conditions subsets  \newchanges{$D_i$ with $N_k$ samples}
\begin{equation}\label{eq:D_i_dataset}
    D_i = \left \{ x_k^{(i)}, y_k^{(i)} = f_k(x_k^{(i)},u_k^{(i)}) + \epsilon_k^{(i)}, x_{r,k}^{(i)}  \right \}_{k=1}^{N_k}
\end{equation}
where $\epsilon_k^{(i)}$ %
is the residual obtained by capturing the discrepancy between the discretized known model dynamics $f_k(x_k^{(i)},u_k^{(i)})$ of \eqref{eq:system_input_affine} and the measured dynamical state.
The Deep Neural Network (DNN) model is based on \changes{$L$} fully connected layers with element-wise ReLU activation function $g(\cdot)=\max (\cdot, 0)$ and the DNN weights $\Theta =\{ W^1,...,W^{L+1}\}$
\begin{equation}
    \phi(x,x_r,\Theta) =  W^{L+1}g(W^L g(\cdots g(W^1 [x^T,x_r^T]^T) \cdots )  )
\end{equation}
Thus, the meta-learning model of the disturbances is 
\begin{equation}
    \hat{f}_d({x}(t),x_r(t)) = \phi({x}(t),x_r(t),{\Theta})\hat{a}(t) \label{eq:f_d_estimate}
\end{equation}

\noindent \newchanges{where $\hat{a}(t)$ is the estimate of $a(w)$.} Define the loss function
\begin{equation}\label{eq:loss}
    \mathcal{L}(\Theta,\{a_i\}_1^M)=\sum_{i=1}^{M}\sum_{k=1}^{N_k} \left \| \epsilon_k^{(i)} - \phi(x_k^{(i)}, x_{r,k}^{(i)}, \Theta)a_i \right \|^2_2\changes{.}
\end{equation}
\changes{Learning is performed using stochastic gradient descent (SGD) and spectral normalization. A parameter $a_i$ is introduced for each dataset to capture the specific environmental conditions. Domain-invariant learning \cite{oconnel2022neural} ensures the neural network basis remains invariant to these conditions. This approach enables us to employ adaptive control to learn and adapt in real-time to the current conditions encoded with the parameter $\hat{a}$.
}
It also enforces  $\left \| \varphi(x,x_d,x_r) \right \| \leq \gamma \left \| x-x_d\right \|$
where $\varphi(x,x_d,x_r) = \phi(x,x_r) - \phi(x_d,x_r)$, and $\gamma$ is specified Lipschitz constant for the spectrally normalized DNN \cite{shi2019neural}.

\subsection{Contraction-Based Adaptive Controller (CBAC)}
\changes{In this section, we present a controller that bounds system behavior to a target trajectory using contraction theory. We first define the error dynamics, then solve a convex optimization problem for the contraction metric, and conclude with a stability theorem and proof.
}
Let us consider the system in \eqref{eq:system_input_affine} and a given target trajectory $(x_d,u_d)$ 
\begin{align}
    \dot{x} &= f(x) + B(x)u + \phi(x,x_r)a + d(x) \label{eq:8} \\
    \dot{x}_d &= f(x_d) + B(x_d)u_d(x_d) + \phi(x_d,x_r)\hat{a}
\end{align}
where $x,x_d:\mathbb{R}_{\geq 0} \rightarrow \mathbb{R}^n$, $u_d:\mathbb{R}^n \rightarrow \mathbb{R}^m$, $f_d(x,x_r)= \phi(x,x_r)a$ captures the interaction-produced and wind disturbances, $f_d(x_d,x_r)=\phi(x_d,x_r)\hat{a}$ are the disturbance estimate used to determine the target trajectory, $d(x)$ are the unmodelled remainder of the bounded disturbances with $\bar{d}=\sup_{x}\left \| d(x)\right \|$ and $f:\mathbb{R}^n \rightarrow \mathbb{R}^n$, $B:\mathbb{R}^n \rightarrow \mathbb{R}^{n\times m}$ are known smooth functions.

Let us define an approximated state-dependent coefficient (SDC) parametrization as $A(x,x_d)$, such that
\begin{equation}
\begin{aligned}
 \label{eq:sdc_inexact}
    f(x)+&B(x)u_d - f(x_d)-B(x_d)u_d = \\
    &A(x,x_d)(x-x_d) + \varepsilon_A(x,x_d)   
\end{aligned}
\end{equation}
where $\varepsilon_A(x,x_d)$ is a parametrization error that can be considered as disturbances $d'(x,x_d) = d(x)+\varepsilon_A(x,x_d)$.
Then, by choosing the control law
\begin{equation}
    u = u_d - K(x,x_d)(x-x_d) - B^{\dagger}(x)\varphi(x,x_d,x_r)\hat{a}
\end{equation}
the dynamics \newchanges{in \eqref{eq:8}} can be equivalently written as
\begin{equation}
\begin{aligned}
    \dot{x} = &\dot{x}_d + (A(x,x_d) - B(x)K(x,x_d))(x-x_d)   \\
    &- B(x)B^{\dagger}(x)\varphi(x,x_d,x_r)\hat{a} + \varphi(x,x_d,x_r)a \label{eq:error_dyn_diff}\\
    &- \phi(x_d,x_r)\tilde{a}+ d'(x)
\end{aligned}
\end{equation}
where $\tilde{a}=\hat{a}-a$ is the error between the estimate $\hat{a}$ and actual parameter $a$, $K(x,x_d)=R^{-1}(x,x_d)B^T(x)M(x,x_d)$ is a state-feedback control gain based on the \changes{contraction} metric $M(x,x_d)$ \changes{\cite{lohmiller1998contraction,tsukamoto2020robust,manchester2017control}} that will be explained later and a weight matrix $R(x,x_d)\succ 0$, where $R(x,x_d)\succ0$ denotes that $R(x,x_d)$ is positive definite, $B^{\dagger}(x) = (B^T(x)B(x))^{-1}B^T(x)$ is the Moore-Penrose inverse of the matrix $B(x)$ which has linearly independent columns and $\varphi(x,x_d,x_r) = \phi(x,x_r) - \phi(x_d,x_r)$.
Note that in such a formulation, the disturbances $f_d(x) = \phi(x,x_r)a$ are matched through two parts. First, the function $\varphi(x,x_d,x_r)$ corresponds to the cancellation of the disturbances based on the discrepancy from the target trajectory $x_d$, and the term $-B(x)B^{\dagger}(x)\varphi(x,x_d,x_r)\hat{a}$ which corresponds to the online adaptation part that acts through the control input and is fundamentally limited via matrix $B(x)$.

\begin{problem}\label{prob:W_convex_optimization}
Let $\underline{\omega}, \bar{\omega} \in (0,\infty)$, $\omega_{\chi}=\bar{\omega}/\underline{\omega}$. Determine \changes{the contraction metric} $M(x,x_d)=W^{-1}(x,x_d)\succ0$, \changes{with $\lambda_{\textup{min}}(W)=\underline{\omega}$, $\lambda_{\textup{max}}(W)=\bar{\omega}$}, by solving the convex optimization problem for a given value of $\alpha\in(0,\infty)$:
\begin{equation}
    \min_{\nu>0,\omega_\chi\in\mathbb{R},\bar{W}\succ 0} \omega_\chi
\end{equation}
subject to the convex constraints 
\begin{gather}
    -\dot{\bar{W}} + 2\textup{sym}(A\bar{W})-2\nu BR^{-1}B^T \preceq -2\alpha \bar{W}, \label{eq:W_C1} \\
    \partial_{b_j(x)} \bar{W} +  \partial_{b_j(x_d)} \bar{W} = 0, \quad j=1,...,m \label{eq:W_C2} \\ %
    I\preceq \bar{W}\preceq \omega_\chi I, %
\end{gather}
where $A=A(x,x_d)$ and $B=B(x)=\left [b_1(x),...,b_m(x)\right ]$ are the state-dependent coefficients defined in \eqref{eq:sdc_inexact}, $\bar{W}=\bar{W}(x)=\nu W(x)$, $\nu=1/\underline{\omega}$, and $\partial_{b_j(x)} \bar{W} = \sum_{i=1}^{n}\frac{\partial W}{\partial x_i}b_{ij}(x)$ is the notation for directional derivative where $b_{ji}(x)$ is the $i$th element of the column vector $b_j$.
\end{problem}
\begin{rem} \label{rem:matrix_m} \changes{We minimize the eigenvalue ratio of $W$ to optimize the transient response and error bounds, while ensuring stability with constraint \eqref{eq:W_C1} and relaxing the need for \textit{a priori} knowledge of the closed-loop controller $K$ \eqref{eq:W_C2}.}
The structure of matrix $M$ is a non-trivial problem that depends on the considered dynamical system. 
It can be computed using SOS programming \cite{manchester2017control} \changes{or, alternatively, pointwise for the state space of interest \cite{tsukamoto2020robust}, with the results then} fitted to a neural network for a representation valid across the entire state space \cite{tsukamoto2020neural}.
In Problem~\ref{prob:W_convex_optimization}, $\alpha$ is \changes{treated} as given, simplifying the optimization. 
For a fixed convergence rate, this approach obtains an appropriate contraction metric and robust set estimates, as shown in Theorem~\ref{thm:main_thm}. However, \changes{obtaining} the optimal contraction rate $\alpha$ and metric $M$ can be achieved \changes{through} a line search on $\alpha$ \cite{singh2023robust}. 
\end{rem}
\begin{thm}\label{thm:main_thm}
Suppose there exists the contraction metric $M(x,x_d)\succ0$ and $M(x,x_d)=W^{-1}(x,x_d)$ obtained by solving Problem~\ref{prob:W_convex_optimization} for a given value of $\alpha\in(0,\infty)$ and that $\sup \| d'(x,x_d,u_d) \| \leq \bar{d}$.
Suppose further that the system is controlled by the following adaptive control law:
    \begin{align}
        u &= u_d - Ke - B^{\dagger}\varphi\hat{a} \\
        \dot{\hat{a}} &= -\sigma \hat{a} + P\phi^TR^{-1}(y-\phi\hat{a}) +P(BB^{\dagger}\varphi)^TMe \label{eq:adapt_law}\\
        \dot{P} &= -2\sigma P + Q - P\phi^TR^{-1}\phi P \label{eq:dotP}
    \end{align}
where $e=x-x_d$, $K=R(x,x_d)^{-1}B(x)^TM(x,x_d)$, $\varphi=\varphi(x,x_d,x_r)$, $\phi = \phi(x,x_r)$, $y$ is the measured discrepancy between the observed error dynamics and the known dynamics, $P$ is the covariance matrix, $Q\succ0$ is a weight matrix and $\sigma \in \mathbb{R}_{\geq 0}$. 
\changes{If there exists $\bar{\alpha}>0$ such that 
\small
\begin{equation}
    -\begin{bmatrix}
2\alpha M  & 2M\phi \\
0 & P^{-1}QP^{-1}+\phi^TR^{-1}\phi \\
\end{bmatrix} \preceq -2\bar{\alpha}\begin{bmatrix}
 M  & 0 \\
0 & P^{-1} \\
\end{bmatrix} \label{eq:LMI}
\end{equation} \normalsize
holds, then} the error dynamics $e$ are bounded with 
\begin{equation*}
    \| e \| \leq \lambda_\mathcal{M}  e^{-\bar{\alpha}t} (\left \| e(0) \right \| + \| \tilde{a}(0)\|) + \left (1- e^{-\bar{\alpha}t} \right ) \frac{D}{\bar{\alpha}\lambda_{\textup{min}}(\mathcal{M})},
\end{equation*}
\newchanges{where}
\begin{equation}\label{eq:D}
    D = \sup_t \left \| \begin{bmatrix}
M(I-BB^{\dagger})\varphi a +Md \\ \phi^TR^{-1}\varepsilon-\sigma P^{-1}{a}-P^{-1}\dot{a}
\end{bmatrix} \right \|
\end{equation} 
\newchanges{and} $\varepsilon = y-\phi a$, \small$\mathcal{M} =  \begin{bmatrix}
 M  & 0 \\
0 & P^{-1} \\
\end{bmatrix}$, \normalsize  $\lambda_\mathcal{M} = \sqrt{\frac{\lambda_{\textup{max}}(\mathcal{M})}{\lambda_{\textup{min}}(\mathcal{M})}}$.

\end{thm}

\begin{proof}
By satisfying conditions \eqref{eq:W_C1} for the matrix $\bar{W}(x,x_d)$, its scaled inverse $M(x,x_d)$ satisfies $\dot{M} + 2\textup{sym}(A_K M) \leq -2\alpha M$ \newchanges{(P1)} \changes{\cite{manchester2017control}}, where $A_K = A-BK$. %
Rewrite the error dynamics $e=x-x_d$ using \eqref{eq:error_dyn_diff} as 
\begin{align*}
\dot{e} &=(A - BK)e - BB^{\dagger}\varphi\hat{a} + \varphi a - \phi\hat{a} + d' \\
&= (A-BK)e -BB^{\dagger}\varphi\tilde{a}+ (I-BB^{\dagger})\varphi a- \phi\hat{a} + d'
\end{align*}
where $\tilde{a}=\hat{a}-a$. Furthermore, $\dot{\tilde{a}}=\dot{\hat{a}}-\dot{a}$ and using \eqref{eq:adapt_law}  
\begin{equation*}
    \dot{\tilde{a}}=-\sigma\tilde{a}-\sigma a + P\phi^TR^{-1}(\varepsilon-\phi\tilde{a}) {-P(BB^{\dagger}\varphi)^TMe}-\dot{a}
\end{equation*}
where $\varepsilon = y-\phi a$.
Similar to \cite[Theorem 2]{tsukamoto2021learning}, we prove the formal stability and robustness guarantees using the contraction analysis.
For a Lyapunov function $V=e^T Me + \tilde{a}^T P^{-1}\tilde{a}$, and using $\frac{\mathrm{d}}{\mathrm{d}t}P^{-1}=-P^{-1}\dot{P}P^{-1}\newchanges{\overset{\eqref{eq:dotP}}{=}}2\sigma P^{-1} - P^{-1}QP^{-1} + \phi^TR^{-1}\phi$ we obtain
\begin{align*}
    \dot{V} &= \begin{bmatrix}
e \\ \tilde{a}
\end{bmatrix}^T\begingroup \setlength\arraycolsep{0pt}
\begin{bmatrix}
\dot{M} + 2\textup{sym}(MA_K)  &  2MBB^{\dagger}\varphi - 2M\phi \\
{-2(BB^{\dagger}\varphi)^TM} & -P^{-1}QP^{-1}-\phi^TR^{-1}\phi \\
\end{bmatrix}\begin{bmatrix}
e \\ \tilde{a}
\end{bmatrix} \endgroup\\
&\phantom{=} + 2\begin{bmatrix}
e \\ \tilde{a}
\end{bmatrix}^T \begin{bmatrix}
M(I-BB^{\dagger})\varphi a +Md' \\ \phi^TR^{-1}\varepsilon-\sigma P^{-1}{a}-P^{-1}\dot{a}
\end{bmatrix} \\
&\overset{\newchanges{(P1)}}{\leq}-\begin{bmatrix} e \\ \tilde{a} \end{bmatrix}^T\begin{bmatrix}
2\alpha M  & 2M\phi \\
0 & P^{-1}QP^{-1}+\phi^TR^{-1}\phi \\
\end{bmatrix}\begin{bmatrix}
e \\ \tilde{a}
\end{bmatrix} \\
&\phantom{=} + 2\begin{bmatrix}
e \\ \tilde{a}
\end{bmatrix}^T \begin{bmatrix}
M(I-BB^{\dagger})\varphi a +Md' \\ \phi^TR^{-1}\varepsilon-\sigma P^{-1}{a}-P^{-1}\dot{a}
\end{bmatrix}
\end{align*}\normalsize
As $P^{-1}QP^{-1}$, $M$ and $P^{-1}$ are all positive definite and uniformly bounded and $\phi^TR^{-1}\phi$ is positive semidefinite, there exists some $\bar{\alpha}>0$ such that \changes{\eqref{eq:LMI} holds} for all $t$ \cite{oconnel2022neural}. 
Then, $
    \dot{V}\leq-2\bar{\alpha}V+2\sqrt{\frac{V}{\lambda_{\textup{min}}(\mathcal{M})}}D,$
where $D$ as in \eqref{eq:D}.
Using the Comparison lemma \cite{Khalil_nonlinear}, and $\left \| \begin{bmatrix}
e \\ \tilde{a}
\end{bmatrix} \right \| \leq \sqrt{\frac{V}{\lambda_{\textup{min}}(\mathcal{M})}}$, we \newchanges{ obtain }
$
    \left \| \begin{bmatrix}
e \\ \tilde{a}
\end{bmatrix} \right \| \leq \lambda_\mathcal{M}  e^{-\bar{\alpha}t} \left \| \begin{bmatrix}
e(0) \\ \tilde{a}(0)
\end{bmatrix} \right \| + \left (1- e^{-\bar{\alpha}t} \right ) \frac{D}{\bar{\alpha}\lambda_{\textup{min}}(\mathcal{M})}.
$ 
The final result follows by $\| e \| \leq \left \| \begin{bmatrix}
e \\ \tilde{a}
\end{bmatrix} \right \| \leq \| e \| + \| \tilde{a}\|$.

\end{proof}

\begin{rem}
A similar problem formulation to Theorem 1 can be found in \cite[Theorem 2]{tsukamoto2021learning}. 
\changes{Our work differs by not assuming} the matched uncertainty condition $\varphi(x,x_d,x_r)a \in \text{span}(B(x))$, \changes{leading to} a less conservative stability theorem valid for a broader class of systems. 
Unmatched disturbances are \changes{addressed} by generating the target trajectory $x_d$, \changes{accounting for} learned disturbances in the optimization problem (Problem~\ref{prob:ocp_standard} \newchanges{below}).
The covariance matrix $P$ for the adaptation variable $\hat{a}$ is updated analogously as in Kalman-Bucy filter \cite{kalman1961bucy}.
\changes{In practice, we implemented the discretized version of CBAC at a frequency of 100 Hz. At this rate, the discretization error is minimal and should not significantly affect performance or stability guarantees. The control law remains unchanged, while the adaptation law follows the two-step Kalman filter approach \cite[Section S4]{oconnel2022neural}}. 
This \changes{result} enables us to further quantify the upper bound \eqref{eq:D} as \changes{done} in Corollary \ref{cor:upper_bound}.
\end{rem}

\subsection{Chance-constrained Upper Bound}
Based on the covariance matrix $P$, and a user-specified small probability of failure $\delta>0$, $\delta\in(0,1)$, we can determine the uncertainty sets as
\begin{equation}
    \mathcal{S}_P(\hat{a},P,\delta):=\{ a :  \| \hat{a} - a \|_{P}^2 \leq \chi_k^2(1-\delta)\}
\end{equation}
where $\chi_k^2(p)$ is the Inverse Cumulative Distribution Function (ICDF) of the chi-square distribution with $k$ degrees of freedom, evaluated at the probability values in $p$.

\begin{cor}\label{cor:upper_bound}
    Assume that the unknown parameter $a$, estimated through the adaptation law \eqref{eq:adapt_law}, varies slowly such that $\dot{a} \approx 0$, and that the estimation $\hat{a}$ has reached a steady state. Then $D$ in \eqref{eq:D} can be upper bounded on a set $x\in \mathcal{X}$ with
    \begin{equation*}
        D\leq \bar{D}:= \frac{\bar{d}}{\omega_\chi} + \bar{\phi}\bar{\varepsilon}\lambda_{\textup{max}}(R) + \left ( \frac{\bar{b}\bar{\phi}}{\omega_\chi}+\lambda_{\textup{min}}(P)\sigma \right ) \sup_t \| a \|,  
    \end{equation*}
    where $\bar{b}=\lambda_{\textup{max}}(I-BB^{\dagger})$, $\bar{\varphi} = \sup_{x\in\mathcal{X}} \| \varphi \| $, $\bar{\phi} = \sup_{x\in\mathcal{X}} \| \phi \| $ and $\bar{\varepsilon} = \sup \|y-f_d\|$ is the upper bound on the measurement noise. Furthermore, a chance-constrained bound can be derived by using
   $\sup_t \| a \| \leq \|\hat{a}\| + \sqrt{\frac{\chi_k^2(1-\delta)}{\lambda_{\textup{min}}(P)}}$.
   
\end{cor}
\begin{rem} 
\changes{The assumption that $\dot{a}$ changes slowly (practically zero) is needed to establish an upper bound that can be calculated, as estimating this parameter beforehand is difficult. 
This assumption holds in practice since wind conditions remain relatively constant during algorithm execution.}
Due to the non-constant terms in the upper bound that depend on the time-varying matrix $P$, the corollary is valid only when the estimation has converged.
\changes{Numerically,} $\bar{D}$ can be computed at each MPC (Section~\ref{ssec:ocp_as_mpc}) time step. 
Finally, the initial adaptation error $\| \tilde{a}(0) \|$ in Theorem~\ref{thm:main_thm}, can be written as $\| \tilde{a}(0) \| \leq \sqrt{\frac{\chi_k^2(1-\delta)}{\lambda_{\textup{min}}(P_{\mathcal{D}_{\textup{meta}}})}}$ \changes{assuming $\hat{a}(0)$ comes from the neural network training step (Section~\ref{ssec:nn_model_of_disturbances}).}
\end{rem}
\begin{cor}\label{cor:error_bound}
    Let $\bar{D}$ be defined as in Corollary~\ref{cor:upper_bound}. Assume that the current wind conditions $a$ have been previously recorded in the dataset $\mathcal{D}_{\textup{meta}}$, and that the initial value for the estimated parameter $\hat{a}(0)$ is determined as described in Section~\ref{ssec:nn_model_of_disturbances}. Then, $\| \tilde{a}(0)\|\leq \sqrt{\frac{\chi_k^2(1-\delta)}{\lambda_{\textup{min}}(P_{\mathcal{D}_{\textup{meta}}})}}$ and the error dynamics can be upper bounded with \small
\begin{equation*}
\begin{aligned}
           \bar{e}(t,\hat{a},P,\delta)=& e^{-\bar{\alpha}t} \left ( \lambda_\mathcal{M} \left \| e(0) \right \| + \lambda_\mathcal{M}\sqrt{\frac{\chi_k^2(1-\delta)}{\lambda_{\textup{min}}(P_{\mathcal{D}_{\textup{meta}}})}} \right )\\
        &+ \left (1- e^{-\bar{\alpha}t} \right ) \frac{\bar{D}}{\bar{\alpha}\lambda_{\textup{min}}(\mathcal{M})},  
\end{aligned}
\end{equation*}\normalsize
\end{cor}
\begin{rem}
\changes{Corrolary~\ref{cor:error_bound} provides a way to compute the error bound, valid for the current conditions, that can be used as an ingredient to obtain safe target trajectories through MPC. This bound certifies safe behaviour near obstacles.} 
 \end{rem}

\subsection{Optimal Control Problem as MPC}\label{ssec:ocp_as_mpc}
We formulate the optimal control problem with respect to the tracking objective $x_r$ to determine the target trajectory $(x_d,u_d)$ which will serve as an input to the CBAC.
\begin{problem}[ML-MPC]\label{prob:ocp_standard}
Let the desired states of the system at time $t$ be $x_d(t)$. Given the reference trajectory $x_r(\cdot|t)$, and the estimated error bound $\bar{e}(t,\hat{a},P,\delta)$ the meta-learning augmented MPC is
\begin{subequations}\label{eq:ml_ocp_problem}
\begin{equation} 
    \min_{{u}(\cdot|t)}
J(\hat{{x}}_d(\cdot|t),{u_d}(\cdot|t),x_r(\cdot|t))    \label{eq:ch6_8} 
\end{equation}
subject to %
\begin{align}
    &    \hat{{x}}_d(k+1|t)=f_k(\hat{{x}}_d(k|t),{u_d}(k|t)),x_r(k|t)),\label{eq:ch6_p_1} \\
    &\hat{{x}}_d(k|t)  \in \mathcal{X}, \\
    &{u}_d(k|t) \in \mathcal{U},\label{eq:ch6_p_3} \\
    &\hat{{x}}_d(k|t) \in \mathcal{X}_{\textup{safe}}(\bar{e}(t_k,\hat{a}(t),P(t),\delta)),\label{eq:ch6_p_2} \\
    &\hat{{x}}_d(N|t) \in \mathcal{X}_{\textup{goal}}(x_r(N|t),\bar{e}(t_N,\hat{a}(t),P(t),\delta)),\label{eq:ocp_goal}
\end{align}
\end{subequations}
for $k=0,1,...,N$, for all $i\in \mathcal{N}$, and $t_k = t+k\Delta t$, where $f_k(x_d,u_d,x_r)$ are discretized dynamics of $\dot{x}_d =f(x_d) + B(x_d)u_d + f_d(x_d,x_r)$.
\end{problem}

Set $\mathcal{X}$ denotes the set of system dynamics state constraints, $\mathcal{U}$ the input constraints, \changes{$\mathcal{X}_{\textup{safe}}(\bar{e}):=\left\{x\in \mathcal{X}: \newchanges{\left\{x \right\}}\cap \left (\mathcal{O} \oplus \mathcal{B}(\bar{e}) \right) = \emptyset \right\}$ 
is a safety set based on the obstacles occupying a closed set $\mathcal{O}$ and $\mathcal{B}(\bar{e})$ is a closed ball of radius $\bar{e}$ centered at the origin in $\mathbb{R}^3$};  $\mathcal{X}_{\textup{goal}}(x_r\newchanges{,\bar{e}}):=\left\{ x_d\in \mathcal{X}_{\textup{safe}}\newchanges{(\bar{e})}: \|x_d-x_r\|\leq \varepsilon_l \right \} $ is the terminal set \changes{around the final point of $x_r$ and $\varepsilon_l>0$ is its radius}. %
We define the cost function as 
\begin{align*}
    J(&\hat{{x}}_d(\cdot|t),{u_d}(\cdot|t),x_r(\cdot|t)) = \| \hat{x}_d(N|t) - x_r(N|t)\|_{Q_{m}}^2 \\
    &+ \sum_{k=0}^{N-1} \| \hat{x}_d(k|t) - x_r(k|t)\|_{Q_{m}}^2 + \| u_d(k|t)\|_{R_{m}}^2
\end{align*}
\changes{where the error between the target trajectory $x_d$ and the tracking objective $x_r$ as well as the control effort are penalized throughout the prediction horizon of length $N$ and matrices $Q_m,R_m\succ0$. 
The MPC formulation in Problem~\ref{prob:ocp_standard} yields a dynamically feasible and optimal target trajectory $(x_d,u_d)$ that accounts for the disturbance model. }

\section{Results}\label{sec:results}
To \changes{evaluate} the proposed algorithm, we use synthetic wind disturbance data and the ground-effect model presented in \cite{danjun2015autonomous}. 
Side disturbances are computed as forces acting on the propellers under constant wind.  
The network is a four-layer fully connected DNN with ReLU activations, which proved to be effective on similar tasks \cite{oconnel2022neural,shi2021neuralswarm2}.
Given the temporal separability between the position and attitude dynamics of the quadrotor model, we design the position controller as CBAC and use a geometric controller valid on complete $SO(3)$ \cite{lee2010geometric, bandyopadhyay2016nonlinear} for the attitude.
The CBAC for the position dynamics of the quadrotor, cast to the form of \eqref{eq:system_input_affine}, is derived by considering the reduced state $x=[p^T,v^T,\eta^T]^T\in\mathcal{X}\subseteq\mathbb{R}^6\times\mathbb{T}$, and input $u=[f_u,\eta_u^T]^T\in\mathcal{U}\subset\mathbb{R}\times\mathbb{T}$ and 
\begin{equation*}
    f(x)=\begin{bmatrix}
        v \\ -ge_3 \\ 0
    \end{bmatrix}  B(x) = \begin{bmatrix}
        0 && 0\\ \frac{1}{m}R(\eta)e_3 && 0 \\ 0 && I_3
    \end{bmatrix} f_d=\begin{bmatrix}
        0 \\ \phi a \\ 0
    \end{bmatrix}
\end{equation*}
where $e_3=[0,0,1]^T$ and $I_3$ is the $3\times3$ identity matrix. One parametrization of the SDC matrix $A(x,x_d,u_d)$ can be obtained by considering Taylor's expansion of $r_3(\eta)=R(\eta)e_3=r_3(\eta_d)+J(\eta_d) \tilde{\eta} + \frac{1}{2}\tilde{\eta}^T H(\eta_d) \tilde{\eta} +o(\| \Delta \eta \|^2)$ where $J(\eta_d) = \frac{\partial r_3}{\partial \eta}(\eta_d)$ is Jacobian matrix, and $H(\eta_d)=\frac{\partial J}{\partial \eta}(\eta_d)$ is Hessian tensor, $\tilde{\eta}=\eta-\eta_d$, and $o(\|\tilde{\eta}\|^2)$ is the little-o notation. 
Thus, for \eqref{eq:sdc_inexact} to hold, we choose
\begin{equation}
    A(x,x_d,u_d) = \begin{bmatrix}
        0 & I_3 & 0 \\
        0 & 0 & \frac{f_u}{m}(J(\eta_d) + \frac{1}{2}\tilde{\eta}^T H(\eta_d)) \\
        0 & 0 & 0
    \end{bmatrix}
\end{equation}
and $\varepsilon_A(x,x_d,u_d) = [0^T,        \frac{f_u}{m}o(\|\tilde{\eta}\|^2)^T,0^T]^T$.
The matrix $M(x,x_d,u_d)$ is obtained by solving the convex optimization problem in Problem~\ref{prob:W_convex_optimization} for a grid of points of the considered state space. We find optimal $\alpha$ using line search and approximate it with the neural network as described in Remark~\ref{rem:matrix_m}.

\begin{figure}[t]
\captionsetup[subfigure]{justification=centering}
    \centering
    %
    %
    %
    %
    %

    \begin{minipage}{0.33\columnwidth}
        \centering
        \includegraphics[width=\columnwidth]{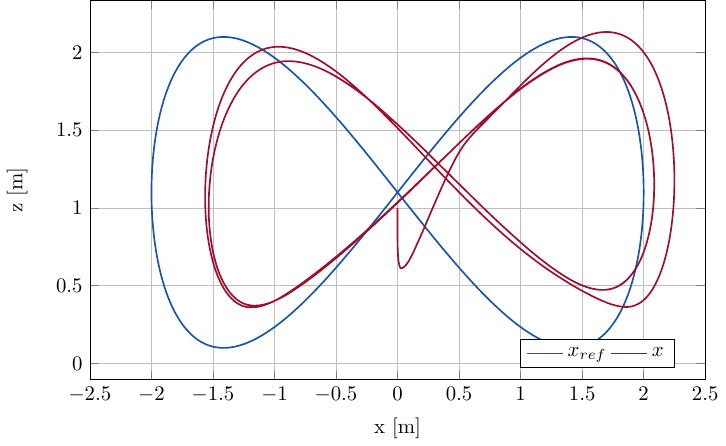}
        \subcaption{Nominal MPC, RMSE: 0.29720}
    \end{minipage}\hfill
    \begin{minipage}{0.33\columnwidth}
        \centering
        \includegraphics[width=\columnwidth]{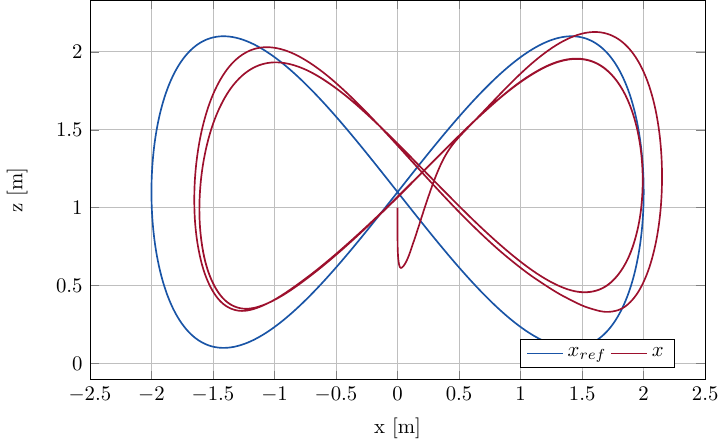}
        \subcaption{\changes{MPC+ML-CBAC, RMSE: 0.22470}}
    \end{minipage}\hfill
    \begin{minipage}{0.33\columnwidth}
        \centering
        \includegraphics[width=\columnwidth]{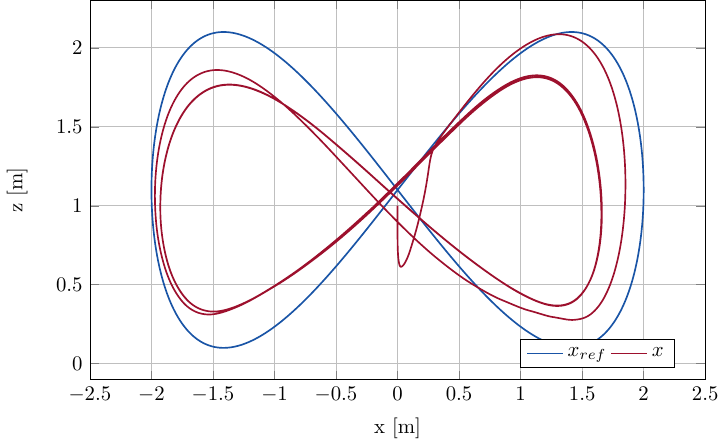}
        \subcaption{Our algorithm, RMSE: 0.17714}
    \end{minipage}

\caption{Performance comparison \changes{of (a) nominal MPC, (b) nominal MPC with a feedback controller consisting of a neural network and an adaptive controller similar to \cite{oconnel2022neural} and (c) our algorithm. For fairness, the NNs in cases (b) MPC without and (c) with disturbance knowledge are identical and trained on the same dataset.} Disturbances include the ground effect %
(ground at $z=0$) and constant crosswind of $12$ m/s. }\label{fig:lemniscate}\vspace{-0.6cm}
\end{figure}

\textit{1) Trajectory Tracking in a Figure-8 Pattern.}
\changes{The framework operates in real-time, achieving 100 Hz for CBAC and 5 Hz for MPC solved with IPOPT in CasADi, with the inference taking at most approximately 10 ms on a personal computer running Ubuntu 22.04 with an Intel i7 processor and an NVIDIA RTX A500 CUDA 12-enabled graphics card.}
Performance is evaluated with Root Mean Square Error (RMSE) relative to the reference trajectory, which is a rotated Figure 8 in the x-z plane. 
When the MPC is unaware of disturbances, discrepancies persist over multiple periods due to the inability of MPC to account for them.
Adding a feedback controller does not improve the performance as it is limited by how well it can track the desired trajectory $x_d$ generated by the MPC. Performance improves when the MPC incorporates knowledge of the disturbance model.
\textit{2) Autonomous Soft Landing.}
We focus on a specific problem of autonomous UAV landing \cite{persson2019MPC,lapandic2021aperiodic,persson2024ecc}. In this setup, we assess whether the algorithm achieves a smooth landing, which means approaching the ground $z=0$ as smoothly as possible. 
Figure~\ref{fig:landing} demonstrates the ability of the algorithm to compensate for the ground effect and land smoothly. It is worth noting that the algorithm in both examples uses the same NN trained on data from the first example.
\begin{figure}
    \centering
    \includegraphics[width=\linewidth]{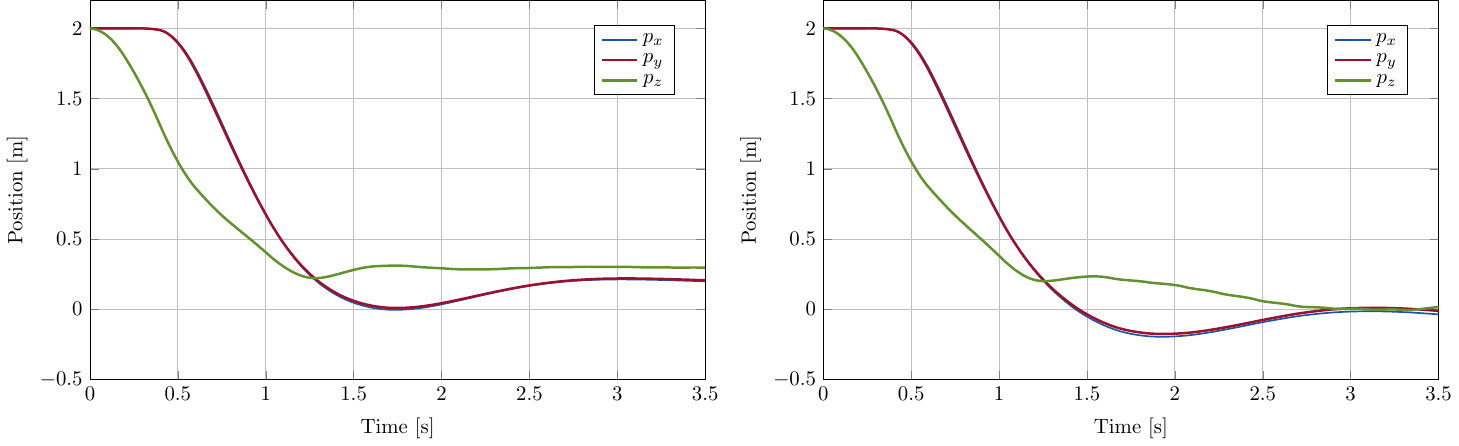}
    \caption{
    (left) Nominal MPC fails to compensate, leaving a $25$ cm error; (right) our algorithm compensates and lands.}
    \label{fig:landing}\vspace{-0.7cm}
\end{figure}
\section{Conclusion}\label{sec:conclusion}
This paper presented a meta-learning augmented MPC algorithm for disturbance-aware motion planning and control, guaranteed to improve performance with respect to desired state-input trajectories. 
The results highlight the importance of disturbance-aware planning for more accurate and reliable behavior. 
By incorporating disturbances into the planning loop, the algorithm improves robustness and trajectory tracking. Future work will focus on testing the control scheme in real-world experiments, extending it to explore other interaction-produced disturbances in unknown environments, and applying it to multi-agent model sharing.%

\vspace{-0.15cm}

\bibliographystyle{unsrt}
\bibliography{refs.bib}

\begin{thebibliography}{10}

\bibitem{bauranov2021designing}
Aleksandar Bauranov and Jasenka Rakas.
\newblock Designing airspace for urban air mobility: A review of concepts and
  approaches.
\newblock {\em Progress in Aerospace Sciences}, 125:100726, 2021.

\bibitem{rajabi2021drone}
Mohammad~Sadra Rajabi, Pedram Beigi, and Sina Aghakhani.
\newblock {\em Drone Delivery Systems and Energy Management: A Review and
  Future Trends}, pages 1--19.
\newblock Springer International Publishing, Cham, 2021.

\bibitem{lyu2023unmanned}
Mingyang Lyu, Yibo Zhao, Chao Huang, and Hailong Huang.
\newblock Unmanned aerial vehicles for search and rescue: A survey.
\newblock {\em Remote Sensing}, 15(13), 2023.

\bibitem{cheeseman1955effect}
IC~Cheeseman and WE~Bennett.
\newblock {The Effect of the Ground on a Helicopter Rotor in Forward Flight}.
\newblock {\em Reports and Memoranda No. 3021, September, 1995}, 1955.

\bibitem{shi2019neural}
Guanya Shi, Xichen Shi, Michael O’Connell, Rose Yu, Kamyar Azizzadenesheli,
  Animashree Anandkumar, Yisong Yue, and Soon-Jo Chung.
\newblock {Neural Lander: Stable Drone Landing Control Using Learned Dynamics}.
\newblock In {\em 2019 International Conference on Robotics and Automation
  (ICRA)}, pages 9784--9790. IEEE, 2019.

\bibitem{aastrom2013adaptive}
Karl~J {\AA}str{\"o}m and Bj{\"o}rn Wittenmark.
\newblock {\em {Adaptive Control}}.
\newblock Courier Corporation, 2013.

\bibitem{lohmiller1998contraction}
Winfried Lohmiller and Jean-Jacques~E Slotine.
\newblock {On Contraction Analysis for Non-Linear Systems}.
\newblock {\em Automatica}, 34(6):683--696, 1998.

\bibitem{hanover2021performance}
Drew Hanover, Philipp Foehn, Sihao Sun, Elia Kaufmann, and Davide Scaramuzza.
\newblock {Performance, Precision, and Payloads: Adaptive Nonlinear {MPC} for
  Quadrotors}.
\newblock {\em IEEE Robotics and Automation Letters}, 7(2):690--697, 2021.

\bibitem{torrente2021data}
Guillem Torrente, Elia Kaufmann, Philipp F{\"o}hn, and Davide Scaramuzza.
\newblock {Data-Driven {MPC} for Quadrotors}.
\newblock {\em IEEE Robotics and Automation Letters}, 6(2):3769--3776, 2021.

\bibitem{faessler2017differential}
Matthias Faessler, Antonio Franchi, and Davide Scaramuzza.
\newblock {Differential Flatness of Quadrotor Dynamics Subject to Rotor Drag
  for Accurate Tracking of High-Speed Trajectories}.
\newblock {\em IEEE Robotics and Automation Letters}, 3(2):620--626, 2017.

\bibitem{sinha2022adaptive}
Rohan Sinha, James Harrison, Spencer~M Richards, and Marco Pavone.
\newblock {Adaptive Robust Model Predictive Control With Matched and Unmatched
  Uncertainty}.
\newblock In {\em 2022 American Control Conference (ACC)}, pages 906--913.
  IEEE, 2022.

\bibitem{khromov2008ground}
Vladimir Khromov and Omri Rand.
\newblock {Ground Effect Modeling for Rotary-Wing Simulation}.
\newblock In {\em 26th International Congress of the Aeronautical Sciences},
  pages 1--10. International Council of the Aeronautical Sciences Germany,
  2008.

\bibitem{he2017modeling}
Xiang He, Marc Calaf, and Kam~K Leang.
\newblock {Modeling and Adaptive Nonlinear Disturbance Observer for Closed-Loop
  Control of in-Ground-Effects on Multi-Rotor {UAV}s}.
\newblock In {\em Dynamic Systems and Control Conference}, volume 58295, page
  V003T39A004. American Society of Mechanical Engineers, 2017.

\bibitem{danjun2015autonomous}
Li~Danjun, Zhou Yan, Shi Zongying, and Lu~Geng.
\newblock {Autonomous Landing of Quadrotor Based on Ground Effect Modelling}.
\newblock In {\em 2015 34th Chinese control conference (CCC)}, pages
  5647--5652. IEEE, 2015.

\bibitem{shi2020neuralswarm}
Guanya Shi, Wolfgang H{\"o}nig, Yisong Yue, and Soon-Jo Chung.
\newblock {Neural-Swarm: Decentralized Close-Proximity Multirotor Control Using
  Learned Interactions}.
\newblock In {\em 2020 IEEE International Conference on Robotics and Automation
  (ICRA)}, pages 3241--3247. IEEE, 2020.

\bibitem{shi2021neuralswarm2}
Guanya Shi, Wolfgang H{\"o}nig, Xichen Shi, Yisong Yue, and Soon-Jo Chung.
\newblock {Neural-Swarm2: Planning and Control of Heterogeneous Multirotor
  Swarms Using Learned Interactions}.
\newblock {\em IEEE Transactions on Robotics}, 38(2):1063--1079, 2021.

\bibitem{oconnel2022neural}
Michael O’Connell, Guanya Shi, Xichen Shi, Kamyar Azizzadenesheli, Anima
  Anandkumar, Yisong Yue, and Soon-Jo Chung.
\newblock {Neural-Fly Enables Rapid Learning for Agile Flight in Strong Winds}.
\newblock {\em Science Robotics}, 7(66), 2022.

\bibitem{10611562}
Fengze Xie, Guanya Shi, Michael O’Connell, Yisong Yue, and Soon-Jo Chung.
\newblock Hierarchical meta-learning-based adaptive controller.
\newblock In {\em 2024 IEEE International Conference on Robotics and Automation
  (ICRA)}, pages 18309--10315, 2024.

\bibitem{richards2021adaptive}
Spencer~M Richards, Navid Azizan, Jean-Jacques Slotine, and Marco Pavone.
\newblock {Adaptive-Control-Oriented Meta-Learning for Nonlinear Systems}.
\newblock {\em arXiv preprint arXiv:2103.04490}, 2021.

\bibitem{harrison2020meta}
James Harrison, Apoorva Sharma, and Marco Pavone.
\newblock {Meta-Learning Priors for Efficient Online Bayesian Regression}.
\newblock In {\em Algorithmic Foundations of Robotics XIII: Proceedings of the
  13th Workshop on the Algorithmic Foundations of Robotics 13}, pages 318--337.
  Springer, 2020.

\bibitem{shi2021meta}
Guanya Shi, Kamyar Azizzadenesheli, Michael O'Connell, Soon-Jo Chung, and
  Yisong Yue.
\newblock {Meta-Adaptive Nonlinear Control: Theory and Algorithms}.
\newblock {\em Advances in Neural Information Processing Systems},
  34:10013--10025, 2021.

\bibitem{mckinnon2021meta}
Christopher~D McKinnon and Angela~P Schoellig.
\newblock {Meta Learning With Paired Forward and Inverse Models for Efficient
  Receding Horizon Control}.
\newblock {\em IEEE Robotics and Automation Letters}, 6(2):3240--3247, 2021.

\bibitem{verginis2023non}
Christos~K Verginis, Zhe Xu, and Ufuk Topcu.
\newblock Non-parametric neuro-adaptive control.
\newblock In {\em 2023 European control conference (ECC)}, pages 1--6. IEEE,
  2023.

\bibitem{10706036}
Elena~Sorina Lupu, Fengze Xie, James~A. Preiss, Jedidiah Alindogan, Matthew
  Anderson, and Soon-Jo Chung.
\newblock Magicvfm -meta-learning adaptation for ground interaction control
  with visual foundation models.
\newblock {\em IEEE Transactions on Robotics}, pages 1--20, 2024.

\bibitem{tsukamoto2020robust}
Hiroyasu Tsukamoto and Soon-Jo Chung.
\newblock {Robust Controller Design for Stochastic Nonlinear Systems via Convex
  Optimization}.
\newblock {\em IEEE Transactions on Automatic Control}, 66(10):4731--4746,
  2020.

\bibitem{manchester2017control}
Ian~R Manchester and Jean-Jacques~E Slotine.
\newblock {Control Contraction Metrics: Convex and Intrinsic Criteria for
  Nonlinear Feedback Design}.
\newblock {\em IEEE Transactions on Automatic Control}, 62(6):3046--3053, 2017.

\bibitem{tsukamoto2020neural}
Hiroyasu Tsukamoto and Soon-Jo Chung.
\newblock {Neural Contraction Metrics for Robust Estimation and Control: A
  Convex Optimization Approach}.
\newblock {\em IEEE Control Systems Letters}, 5(1):211--216, 2020.

\bibitem{singh2023robust}
Sumeet Singh, Benoit Landry, Anirudha Majumdar, Jean-Jacques Slotine, and Marco
  Pavone.
\newblock {Robust Feedback Motion Planning via Contraction Theory}.
\newblock {\em The International Journal of Robotics Research}, 42(9):655--688,
  2023.

\bibitem{tsukamoto2021learning}
Hiroyasu Tsukamoto, Soon-Jo Chung, and Jean-Jacques Slotine.
\newblock {Learning-Based Adaptive Control Using Contraction Theory}.
\newblock In {\em 2021 60th IEEE Conference on Decision and Control (CDC)},
  pages 2533--2538. IEEE, 2021.

\bibitem{Khalil_nonlinear}
Hassan~K Khalil.
\newblock {\em {{Nonlinear Systems; 3rd Ed.}}}
\newblock Prentice-Hall, 2002.

\bibitem{kalman1961bucy}
R.~E. Kalman and R.~S. Bucy.
\newblock {New Results in Linear Filtering and Prediction Theory}.
\newblock {\em Journal of Basic Engineering}, 83(1):95--108, 03 1961.

\bibitem{lee2010geometric}
Taeyoung Lee, Melvin Leok, and N~Harris McClamroch.
\newblock {Geometric Tracking Control of a Quadrotor {UAV} on {SE}(3)}.
\newblock In {\em 49th IEEE conference on decision and control (CDC)}, pages
  5420--5425. IEEE, 2010.

\bibitem{bandyopadhyay2016nonlinear}
Saptarshi Bandyopadhyay, Soon-Jo Chung, and Fred~Y Hadaegh.
\newblock {Nonlinear Attitude Control of Spacecraft With a Large Captured
  Object}.
\newblock {\em Journal of Guidance, Control, and Dynamics}, 39(4):754--769,
  2016.

\bibitem{persson2019MPC}
Linnea Persson and Bo~Wahlberg.
\newblock {\em Model Predictive Control for Autonomous Ship Landing in a Search
  and Rescue Scenario}.
\newblock AIAA, 2019.

\bibitem{lapandic2021aperiodic}
Dženan Lapandić, Linnea Persson, Dimos~V. Dimarogonas, and Bo~Wahlberg.
\newblock {Aperiodic Communication for {MPC} in Autonomous Cooperative
  Landing}.
\newblock {\em IFAC-PapersOnLine}, 54(6):113--118, 2021.
\newblock 7th IFAC Conference on Nonlinear Model Predictive Control {NMPC}
  2021.

\bibitem{persson2024ecc}
Linnea Persson, Anders Hansson, and Bo~Wahlberg.
\newblock An optimization algorithm based on forward recursion with
  applications to variable horizon mpc.
\newblock {\em European Journal of Control}, 75:100900, 2024.

\end{thebibliography}

\end{document}